\documentclass{article}

\PassOptionsToPackage{numbers, compress}{natbib}

\usepackage[preprint]{neurips_2024}

\usepackage[ruled,linesnumbered,lined,boxed,commentsnumbered]{algorithm2e}
\newcommand*{\dif}{\mathop{}\!\mathrm{d}}

\usepackage[utf8]{inputenc} 
\usepackage[T1]{fontenc}    
\usepackage{hyperref}       
\usepackage{url}            
\usepackage{booktabs}       
\usepackage{amsfonts}       
\usepackage{nicefrac}       
\usepackage{microtype}      
\usepackage{xcolor}         
\usepackage{amsmath}
\usepackage{amssymb}
\usepackage{mathtools}
\usepackage{amsthm}
\usepackage[capitalize,noabbrev]{cleveref}
\usepackage{graphicx}  
\usepackage{subcaption}  
\usepackage{float}

\theoremstyle{plain}
\newtheorem{theorem}{Theorem}[section]

\newtheorem{lemma}[theorem]{Lemma}

\theoremstyle{definition}

\theoremstyle{remark}

\DeclareMathOperator*{\argmin}{arg\,min}

\title{Towards Differentiable Multilevel Optimization: A Gradient-Based Approach}

\author{%
  Yuntian Gu \quad Xuzheng Chen\\
  School of Intelligence Science and Technology, Peking University\\ 
  \fontsize{9pt}{0pt}{\texttt{guyuntian@stu.pku.edu.cn} \quad \texttt{cxz1022@stu.pku.edu.cn}}
} 

\begin{document}

\maketitle

\begin{abstract}
Multilevel optimization has recently gained renewed interest in machine learning due to its promise in applications such as hyperparameter tuning and continual learning. However, existing methods struggle with the inherent difficulty of handling the nested structure efficiently. This paper introduces a novel gradient-based approach for multilevel optimization designed to overcome these limitations by leveraging a hierarchically structured decomposition of the full gradient and employing advanced propagation techniques. With extension to $n$-level scenarios, our method significantly reduces computational complexity and improves solution accuracy and convergence speed. We demonstrate the efficacy of our approach through numerical experiments comparing with existing methods across several benchmarks, showcasing a significant improvement of solution accuracy. To the best of our knowledge, this is one of the first algorithms providing a general version of implicit diﬀerentiation with both theoretical guarantee and outperforming empirical superiority.
\end{abstract}

\section{Introduction}\label{intro}

Multilevel optimization, comprising problems where decision variables are optimized at multiple hierarchical levels, has been widely studied in economics, mathematics and computer science. Its special case, Bilevel optimization, was first introduced by Stackelberg in 1934 \citep{stackelberg1952theory, dempe2020bilevel} as a concept in economic game theory  and then well established by following works. Recently, with the rapid development of deep learning and machine learning techniques, multilevel optimization has regained attention and been widely exploited in reinforcement learning \citep{hong2023two}, hyperparameter optimization \citep{franceschi2017forward, lorraine2020optimizing}, neural architecture search \citep{liu2018darts}, and meta-learning \citep{franceschi2018bilevel, rajeswaran2019meta}. These works demonstrate the potential of multilevel optimization framework in solving diverse problems with underlying hierarchical dependency within modern machine learning studies.

Despite the proliferated literature, many of the existing works focus on the simplest case, bilevel optimization, due to the challenges of inherent complexity and hierarchical interdependence in multilevel scenarios. Typical approaches to bilevel optimization include constraint-based and gradient-based algorithms \citep{liu2021investigating, zhang2024introduction}, as well as transforming the original problem into equivalent single-level ones. Among them, gradient based techniques have became the most popular strategies \citep{maclaurin2015gradient, franceschi2018far}. Specifically, \citet{franceschi2017forward} first calculates gradient flow of the lower level objective, and then performs gradient computations of the upper level problem. Additionally, value function based approach \citep{liu2023value} offers flexibility in solving non-singleton lower-level problems, and single-loop momentum-based recursive bilevel optimizer \citep{yang2021provably} demonstrates lower complexity than existing stochastic optimizers. 

Extending from the majority works on bilevel cases, more and more attempts have been made to tackle trilevel or multilevel optimization problems \citep{tilahun2012new, shafiei2021trilevel}. 
Yet the adoption of these methods is still hindered by significant challenges. For one, the application of gradients within multilevel optimization, which involves the intricate process of chaining best-response Jacobians, demands a high level of both computational and mathematical skills and arduous efforts in engineering. Moreover, the computational frameworks for calculating these Jacobians, such as iterative differentiation (ITD) \citep{sato2021gradient, lorraine2020optimizing}, is resource-heavy and complex, thus obstructing the implementation of practical applications. Therefore, it calls for more concise, clear, and easy to implement methods with high efficiency in the field of multilevel optimization.

This paper takes a step towards efficient and effective differentiable multilevel optimation. We propose an effective first-order approach solving the upper problem in multilevel optimization (Section \ref{sec:alg}). We then show our method goes beyond trilevel optimization to $n$-level problem by induction, where we provide a concise algorithm for detailed procedures. We rigorously prove the theoretical guarantee of convergence for our method (Section \ref{sec:theo}). Given premises of the boundness of Hessian matrix, we achieve an order of $\mathcal{O}(1/N)$ convergence rate for multilevel optimization when the domain is compact and convex and derive more subtle results of $\mathcal{O}(n^n/N)$ in general. This is a substantial leap forward in that previous works either do not have convergence guarantee \citep{liu2018darts} or only converges asymptotically \citep{sato2021gradient}. We also demonstrate the superiority of our implicit differentiation methods with numerical experiments (Section \ref{sec:exp}). As the result shows, our method achieves $3.3$ times faster speed than ITD in computational efficiency, and better convergence results on the original benchmark.

\section{Background}

\subsection{Bilevel Optimization and Solution Approaches}\label{subsec:bi}
When modeling real-world problems, there is often a hierarchy of decision-makers, and decisions are taken at different levels in the hierarchy. This procedure can be modeled with multilevel optimization. We start by inspecting its simplest case, bilevel optimization, which has been well explored and inspired our approach to multilevel settings. The general form of bilevel optimization is 
\begin{equation}
\label{base}
\begin{aligned}
    &\text{minimize}_{\bf{x} \in S}  &f^U(\bf{x}, \bf{y})  \\
    &\text{subject to}  &\bf{y} \in \argmin_{\bf{y^\prime}} \it{f^L} (\bf{x},\bf{y^\prime})
\end{aligned}
\end{equation}
where U and L mean upper level and lower level respectively, $\bf{x} \in S \subset \mathbb{R}^{m}$, $\bf{y} \in \mathbb{R}^{n}$, and $f^U, f^L: \mathbb{R}^{m+n} \to \mathbb{R}$. This process can be viewed as a leader-follower game. The lower level follower always optimizes its own utility function $f^L$ based on the leader's action. And the upper level leader will determine its optimal action for utility function $f^U$ with the knowledge of follower's policy.

Generally, the solutions to bilevel optimization can be classified into three categories. The first approach tries to explicitly derive the function of $y$ with respect to $x$ from the lower level optimization problem, turning the upper level optimization into unconstrained optimization $\min_{\bf{x}}  f^U(\bf{x}, y(\bf{x}))$,
which is easy to solve through analytic methods or numerical methods like gradient decent. 

The second approach solves the problem by replacing the 
lower level problem with equivalent forms, like inequalities or KKT conditions. For example, the \cref{base} can be equivalently represented as $\min_{\bf{x}}  f^U (\bf{x}, \bf{y})$ subject to $f^L(\mathbf{x}, \mathbf{y}) \le (f^L)^*$, where $(f^L)^*$ is the minimum of $f^L(\bf{x}, \bf{y})$.


The last approach is gradient-based, which draws our attention. The gradient of the upper level function $f^U(\bf{x}, \bf{y})$ with respect to $\bf{x}$ can be written as $ \frac{\dif f^U}{\dif \bf{x}} = \frac{\partial f^U}{\partial \bf{x}} + \frac{\partial f^U}{\partial \bf{y}} \frac{\dif \bf{y}}{\dif \bf{x}}$.
Since $\frac{\partial f^U}{\partial \bf{x}}$ and $\frac{\partial f^U}{\partial \bf{y}}$ are known, if we can calculate (or approximate) the gradient of $\bf{y}$ to $\bf{x}$, i.e. $\frac{d \bf{y}}{d \bf{x}}$, we will be able to optimize the upper level function through gradient decent. Note that the gradient is calculated at $(\bf{x}, \bf{y}^*)$ due to the lower level constraint. The method we propose in the following section will provide how to calculate $\frac{\dif f^U}{\dif \mathbf{x}}$ in the case of trilevel and $n$-level.

\subsection{Implicit Differentiation Method}
As mentioned above, we can see that the key to getting the gradient is to obtain the derivative of $\bf y$ with respect to $\bf x$, because $f$ is known and the partial derivatives of $f$ with respect to $\bf x$ and $\bf y$ can be calculated easily. One intuitive idea is to utilize the function $f^L(\bf{x}, \bf{y})$ to get the relationship between $\bf x$ and $\bf y$ using chain rule \citep{samuel2009learning}: we know that the constraint that we need to minimize $f^L(\bf{x}, \bf{y})$ can provide the relationship $\bf{y} = \bf{y}(\bf{x})$ so that we can replace $\bf{y}$ with $\bf{x}$ in $f^U$. 

The derivative of $f^L$ may be an unsolvable differential equation and may not directly provide an explicit expression $\bf{y} = \bf{y}(\bf{x})$. \citet{gould2016differentiating} proposes implicit differentiation method to get the derivative without solving the differential equation as follows:

\begin{lemma}\label{lemma_31}
Let $f:\mathbb R^{m+n}\to \mathbb R$ be a continuous function with second derivatives and $\det (\frac{\partial ^2 f}{\partial \bf{y} ^2} ) \neq 0$. Let $\bf{g}(\bf{x}) = \argmin _{\bf y} f(\bf x, \bf y)$. Then the derivative of $\bf g$ with respect to $\bf x$ is
$$\dfrac{\dif \mathbf{g}(\mathbf{x})}{\dif \mathbf{x}} = - f_{YY}(\mathbf{x}, \mathbf{g}(\mathbf{x}))^{-1} f_{XY}(\mathbf{x}, \mathbf{g}(\mathbf{x}))$$
\end{lemma}
Detailed proofs of lemma \ref{lemma_31} can be seen in Appendix \ref{2.1}.

Based on the lemmas above, we have a method to calculate the derivative of $\bf y$ with respect to $\bf x$, which allows us to obtain gradients easily. In the follow-up of the paper, we will discuss higher-level frameworks, mainly focusing on trilevel and extending to $n$-level conclusions.

\subsection{Procedures and Mathematical Formulation of Multilevel Optimization}
As discussed in \ref{subsec:bi}, bilevel optimization can be viewed as a leader-follower game. The follower makes his own optimal decision $\bf y(x)$ for each fixed decision $\bf x$ of the leader. With the optimal decision $\bf y(x)$ of the follower, the leader can select $\bf x$ that maximizes his utility, which makes $f(\bf x, y)$ to the minimum here. 

Essentially, this can be considered as a special case of sequential game with two players. In sequential games, players choose strategies in a sequential order, hence some may take actions first while others later. When there are $3$ players making decisions in a specific order, it can be seen as leaders of different levels. Specifically, the third player takes $\argmin_{\bf z} f_3(\bf x,y,z)$ given fixed $\bf x$ and $\bf y$. The second player takes $\argmin_{\bf y} f_2(\bf x,y,z)$ given fixed $\bf x$, with perfect understanding of how $\bf z$ will change according to $\bf y$. The first player gives $\argmin_{\bf x} f_1(\bf x,y,z)$, knowing fully how $\bf y$ and $\bf z$ change according to $\bf x$.

Generally, the mathematical formulation of $n$-level optimization problem can be expressed as follows:
$$
\begin{aligned}
 \min _{\bf x_1 \in S} f_1 (\bf x_1, & \bf x_2^*, \ldots \bf x_n^*)  \\
\text {s.t. } \bf{x}_2^* & =  \argmin _{\bf x_2} f_2\left(\bf x_1, \bf x_2, \ldots, \bf x_n^*\right) \\
\quad \quad \cdots & \\
& \text { s.t. } \bf{x_n}^*=\argmin _{\bf x_n} f_n\left(\bf x_1, \bf x_2, \ldots \bf x_n\right) \\
\end{aligned}
$$
Given the formal definition of multilevel optimization, we will propose our method for computing the full gradient of upper-level problems, and provide theoretical analysis along with experimental results in the following sections.

\section{Gradient-based Optimization Algorithm for Multilevel Optimization}\label{sec:alg}

In this section, we will propose a new algorithm for trilevel optimization, and extend the results to general multilevel optimization.

\subsection{The gradient of Trilevel Optimization}

We start by considering the solution of lower level problems 
 \begin{align*}
      \bf{g(x,y)} &= \argmin_{\bf{z} \in \mathbb{R}^{d_3}}f_3(\bf{x,y,z})\\
      \bf{h(x)} &= \argmin_{\bf{y} \in \mathbb{R} ^{d_2}} f_2(\bf{x,y,g(x,y)})
 \end{align*}
In this work, we focus on the case where the lower-level solution $\bf g, h$ are singletons, which covers a variety of learning tasks \citep{domke2012generic, mackay2019self}. We further assume that $f_2$ and $f_3$ are strictly convex so that $\det (\frac{\dif ^2 f_3}{\dif \bf{z} ^2}) > 0$ and $\det (\frac{\dif ^2 f_2}{\dif \bf{y} ^2}) > 0$, but do not require $\bf g(x,y)$ or $\bf h(x)$ to have a closed-form formula. Throughout the remainder of the paper, we will evaluate all derivatives at the optimal solution of the lower-level subproblem. For the sake of brevity, we will use notations such as $f_{XY}$ to denote $f_{XY}(\bf{x}, \bf{h(x)}, \bf{g(x, y)})$.

\begin{lemma}
\label{lemma1}
Let $f: \mathbb{R}^{d_1+d_2+d_3} \to \mathbb{R}$ be a continuous function with second derivatives. Let $\mathbf{g(x,y)} = \argmin_{\mathbf{z} \in \mathbb{R}^{d_3}}f(\mathbf{x,y,z})$, then the following properties about $g$ holds:\\
(a)
$$\dfrac{\partial \mathbf{g}}{\partial \mathbf{x}} = -f_{ZZ}^{-1}f_{XZ} $$
(b)
$$\dfrac{\partial \mathbf{g}}{\partial \mathbf{y}} = -f_{ZZ}^{-1}f_{YZ} $$
\end{lemma}

\begin{lemma}
\label{lemma2}
Let $f: \mathbb{R}^{d_1+d_2+d_3} \to \mathbb{R}$ and $\mathbf{g}: \mathbb{R}^{d_1+d_2} \to \mathbb{R}^{d_3}$ be continuous functions with second derivatives. Let $\mathbf{h(x)} = \argmin_{y\in \mathbb{R}}f(\mathbf{x,y,g(x,y)})$, then the derivative of $\bf h$ with respect to $\bf x$ is:
\begin{multline*}
    \dfrac{\partial \mathbf{h}}{\partial \mathbf{x}} = -\left(f_{YY} + 2f_{ZY} \dfrac{\partial \mathbf{g}}{\partial \mathbf{y}}  + (\dfrac{\partial \mathbf{g}}{\partial \mathbf{y}})^\top f_{ZZ} \dfrac{\partial \mathbf{g}}{\partial \mathbf{y}} + \sum _{i=1}^{d_3} (f_Z)_i \dfrac{\partial}{\partial \mathbf{y}} (\dfrac{\partial \mathbf{g}}{\partial \mathbf{y}})_i\right)^{-1}\\
    \left(f_{XY} + f_{ZY} \dfrac{\partial\mathbf{g}}{\partial \mathbf x}  + (\dfrac{\partial \mathbf g}{\partial \mathbf y})^\top f_{XZ} +  (\dfrac{\partial \mathbf g}{\partial \mathbf y})^\top f_{ZZ}\dfrac{\partial \mathbf g}{\partial \mathbf x} + \sum_{i=1}^{d_3} (f_Z)_i \dfrac{\partial}{\partial \mathbf x}(\dfrac{\partial \mathbf g}{\partial \mathbf y})_i\right )
\end{multline*}
\end{lemma}

The basic idea of the proofs of lemma \ref{lemma1} and lemma \ref{lemma2} is to take the derivative of $\frac{\dif \mathbf g}{\dif \mathbf z}$ and $\frac{\dif \mathbf h}{\dif \mathbf y}$, and we provide the detailed proofs in Appendix \ref{3.1}.

Now, we are ready to bring out the main theorem, which gives the gradient of the upper problem and enable us with all kinds of gradient-based optimization algorithm like Adam \citep{kingma2017adam}.

\begin{theorem}
\label{theorem3}
Define the objective functions $f_1, f_2, f_3: \mathbb{R}^{d_1+d_2+d_3} \to \mathbb{R}$ with first-order, second-order and third-order derivatives respectively. Define the functions $g$ and $h$ as:
\begin{align*}
      \bf{g(x,y)} &= \argmin_{\bf{z} \in \mathbb{R}^{d_3}}f_3(\bf{x,y,z})\\
      \bf{h(x)} &= \argmin_{\bf{y} \in \mathbb{R} ^{d_2}} f_2(\bf{x,y,g(x,y)})
\end{align*}
Then the derivative of $f_1$ with respect to $\bf x$ can be written as:
$$\dfrac{\dif f_1}{\dif \mathbf{x}} = \dfrac{\partial f_1}{\partial \mathbf{x}} + (\dfrac{\partial \mathbf{h}}{\partial \mathbf{x}})^\top \dfrac{\partial f_1}{\partial \mathbf y} + (\dfrac{\partial \mathbf g}{\partial \mathbf x} + \dfrac{\partial \mathbf g}{\partial \mathbf y} \dfrac{\partial \mathbf h}{\partial \mathbf x})^\top \dfrac{\partial f_1}{\partial \mathbf z}$$
where all the derivatives are calculated at $\mathbf{y=h(x), z=g(x, h(x))}$. The derivative of $\mathbf g$ and $\mathbf h$ can be calculated using lemma \ref{lemma1} and lemma \ref{lemma2}.
\end{theorem}

\cref{theorem3} is the direct result of composing best-response Jacobians via applying chain rule twice. 

\subsection{The gradient of Multilevel Optimization}
Although by far the discussion has focused on trilevel optimization, our method is not confined to this scope. Previously, we consolidated all relevant information from the lower-level problem, then applied implicit differentiation to the upper-level problem. Next, we propose an algorithm that employs recursion to address the general multilevel optimization problem.

\paragraph{High-level idea}
Consider an \(n\)-level optimization problem where \(f_1, f_2, \ldots, f_n\) are smooth functions. Assuming the existence of an algorithm capable of solving the \((n-1)\)-level problem, which yields solutions for \(\frac{\dif \mathbf{x}_i}{\dif \mathbf{x}_j}\) and \(\frac{\partial \mathbf{x}_i}{\partial \mathbf{x}_j}\) for every \(2 \leq j < i\), we can apply the chain rule as follows:
\begin{equation}
\begin{aligned}
\label{eq:high}
\dfrac{\dif \mathbf{x}_i}{\dif \mathbf{x}_1} = \sum _{1\leq j<i}\dfrac{\partial \mathbf{x}_i}{\partial \mathbf{x}_j} \dfrac{\dif \mathbf{x}_j}{\dif \mathbf{x}_1} = \sum _{P\in P^{(1,i)}} \prod _{j=1}^{|P| - 1} \dfrac{\partial \mathbf{x}_{P_{j+1}}}{\partial \mathbf{x}_{P_j}} = \sum _{2\leq j\leq i}\dfrac{\dif \mathbf{x}_i}{\dif \mathbf{x}_j} \dfrac{\partial \mathbf{x}_j}{\partial \mathbf{x}_1}
\end{aligned}
\end{equation}

where \(P^{(1,i)}\) represents the set of paths from \(1\) to \(i\), with all instances of multiplication referring to the matrix multiplication of Jacobian matrices.

It is noteworthy that \(\frac{\partial \mathbf{x}_i}{\partial \mathbf{x}_1}\) for \(i > 2\) can be computed by solving an \((n-1)\)-level problem, by treating \(\mathbf{x}_2\) as a constant. Consequently, the primary challenge that remains is determining \(\frac{\partial \mathbf{x}_2}{\partial \mathbf{x}_1}\). Once we have the partial derivative \(\frac{\partial \mathbf{x}_2}{\partial \mathbf{x}_1}\), the full derivative $\frac{\dif \mathbf{x}_i}{\dif \mathbf{x}_1}$ of any $i$ can be easily calculated via \cref{eq:high}.

As \(\mathbf x_2\) minimizes \(f_2\), we denote this condition as \(f_2' = \frac{\dif f_2}{\dif \mathbf x_2} = 0\). Consequently,
\begin{equation}
\begin{aligned}
\label{solve}
\dfrac{\dif f_2'}{\dif \mathbf x_1} &= \dfrac{\partial f_2'}{\partial \mathbf x_1} + \dfrac{\partial f_2'}{\partial \mathbf x_2}\dfrac{\partial \mathbf x_2}{\partial \mathbf x_1} + \sum _{i=3}^n \dfrac{\partial f_2'}{\partial \mathbf x_i}\dfrac{\dif \mathbf x_i}{\dif \mathbf x_1}\\
&= \dfrac{\partial f_2'}{\partial \mathbf x_1} + \dfrac{\partial f_2'}{\partial \mathbf x_2}\dfrac{\partial \mathbf x_2}{\partial \mathbf x_1} + \sum _{i=3}^n \dfrac{\partial f_2'}{\partial \mathbf x_i} \left( \dfrac{\dif \mathbf x_i}{\dif \mathbf x_2} \dfrac{\partial \mathbf x_2}{\partial \mathbf x_1}  +\sum _{j=3}^i \dfrac{\dif \mathbf x_i}{\dif \mathbf x_j} \dfrac{\partial \mathbf x_j}{\partial \mathbf x_1} \right)\\
&= \dfrac{\dif f_2'}{\dif \mathbf x_2}\dfrac{\partial \mathbf x_2}{\partial \mathbf x_1} + \dfrac{\partial f_2'}{\partial \mathbf x_1} + \sum _{i=3}^n \dfrac{\partial f_2'}{\partial \mathbf x_i} \sum _{j=3}^i \dfrac{\dif \mathbf x_i}{\dif \mathbf x_j} \dfrac{\partial \mathbf x_j}{\partial \mathbf x_1} = 0
\end{aligned}
\end{equation}

Therefore, \(\frac{\partial \mathbf x_2}{\partial \mathbf x_1}\) can be derived as:
\begin{equation}
\dfrac{\partial \mathbf x_2}{\partial \mathbf x_1} = -(\dfrac{\dif f_2'}{\dif \mathbf x_2})^{-1}(\dfrac{\partial f_2'}{\partial \mathbf x_1} + \sum _{i=3}^n \dfrac{\partial f_2'}{\partial \mathbf x_i} \sum _{j=3}^i \dfrac{\dif \mathbf x_i}{\dif \mathbf x_j} \dfrac{\partial \mathbf x_j}{\partial \mathbf x_1})
\end{equation}

 Furthermore, if \(f_2\) is strictly convex in \(\mathbf x_2\), \(\frac{\dif f_2'}{\dif \mathbf x_2}\) represents the Hessian matrix, which is assured to be positive definite. Consequently, the equation can be resolved by applying matrix inverse to the Hessian, which can be accelerated via many algorithms like conjugated gradients \citep{hestenes1952methods}.

\paragraph{Algorithm} We introduce an algorithm (see Algorithm \ref{alg}) designed to compute \(\frac{\dif \mathbf x_i}{\dif \mathbf x_j}\) and \(\frac{\partial \mathbf x_i}{\partial \mathbf x_j}\) for all \(i\) and \(j\). This enables us to determine the gradient \(\frac{\dif f_1}{\dif \mathbf x_1}\) using the chain rule:
\begin{equation}\label{n_var_deriv}
    \dfrac{\dif f_1}{\dif \mathbf x_1} = \sum\limits_{j=1}^n \dfrac{\partial f_1}{\partial \mathbf x_j}\dfrac{\dif \mathbf x_j}{\dif \mathbf x_1}
\end{equation}

\begin{algorithm}
\caption{Computation of $\frac{\dif \mathbf x_i}{\dif \mathbf x_j}, \frac{\partial \mathbf x_i}{\partial \mathbf x_j}$}
\label{alg}
\SetKwInOut{Input}{Input}\SetKwInOut{Output}{Output}
\Input{Current value of the $1$-st level problem $\mathbf x_1$}
\Input{Approximate solution for $(n-1)$th level problem $ \mathbf x_2, \mathbf x_3, \ldots , \mathbf x_n$} 
\Input{Objective function $f_2$}
Fix $\mathbf x_1$, apply $(n-1)$-level \cref{alg} to get: $\frac{\dif \mathbf x_i}{\dif \mathbf x_j}, \frac{\partial \mathbf x_i}{\partial \mathbf x_j}, \forall 2\leq j < i$\;
Fix $\mathbf x_2$, apply $(n-1)$-level \cref{alg} to get: $\frac{\partial \mathbf x_i}{\partial \mathbf x_1}, \forall i \geq 3$\;
$f_2' := \sum_{i=2}^n \frac{\partial f_2}{\partial \mathbf x_i} \frac{\dif \mathbf x_i}{\dif \mathbf x_2}$\;
$f_2'' := \sum_{i=2}^n \frac{\partial f_2'}{\partial \mathbf x_i} \frac{\dif \mathbf x_i}{\dif \mathbf x_2}$\;
$A:= \frac{\partial f_2'}{\partial \mathbf x_1}$\;
\For{$i\leftarrow 3$ \KwTo $n$}{
$A = A + \frac{\partial f_2'}{\partial \mathbf x_i} \sum _{j=3}^i \frac{\dif \mathbf x_i}{\dif \mathbf x_j}\frac{\partial \mathbf x_j}{\partial \mathbf x_1}$\;
}
$\frac{\partial \mathbf x_2}{\partial \mathbf x_1} =\frac{\dif \mathbf x_2}{\dif \mathbf x_1} = -(f_2'')^{-1}A$\;
\For{$i\leftarrow 3$ \KwTo $n$}{
$\frac{\dif \mathbf x_i}{\dif \mathbf x_1} = \sum _{j=2}^i \frac{\dif \mathbf x_i}{\dif \mathbf x_j}\frac{\partial \mathbf x_j}{\partial \mathbf x_1} $\;
}
\end{algorithm}

\section{Theoretical Analysis}\label{sec:theo}
\subsection{Complexity of Calculating Gradients}
Denote $d_i$ as the dimension of vector $\mathbf x_i$, and $d = \max (d_1,...,d_n)$. The calculation of $\frac{\dif f_1}{\dif \mathbf x_1}$ involves calling \cref{alg} and applying the chain rule.

\paragraph{Complexity Analysis of Algorithm \ref{alg}} Let the time complexity of solving an \(n\)-level problem be denoted by \(F(n)\). The primary computational demand of this algorithm arises from the operations in lines 7 and 11, each of which involves matrix multiplication with a complexity of \(\mathcal{O}(d^3)\). This contributes to an overall complexity of \(\mathcal{O}(d^3n^2)\). A straightforward implementation of the \(n\)-level algorithm would naively invoke the \((n-1)\)-level algorithm twice, leading to an exponential increase in complexity. However, when invoking Algorithm \ref{alg} for the second time in line 2, it is unnecessary to repeat the computation of line 1 since its result is already available. Let the time complexity associated with line 2 for an \(n\)-level problem be \(G(n)\).
\begin{align*}
G(n) &= G(n-1) + \mathcal{O}(d^3n^2),\\
F(n) &= F(n-1) + G(n-1) + \mathcal{O}(d^3n^2).
\end{align*}
From this, we deduce that \(F(n) = \mathcal{O}(d^3n^4)\), which reflects the adjusted complexity taking into account the optimization in line 2.

\paragraph{Complexity Analysis of Chain Rule Calculation} The computation of the chain rule necessitates \(n-1\) instances of matrix-vector multiplication. Each of these operations carries a complexity of \(\mathcal{O}(d^2)\), leading to an overall complexity of \(\mathcal{O}(d^2n)\) for the chain rule component.

When considering the total computational expense of computing the gradient, it amounts to \(\mathcal{O}(d^3n^4)\). This represents a significant efficiency improvement over the method outlined in \citet{sato2021gradient}, which is characterized by a complexity of \(\mathcal{O}(d^n n!)\). This comparison highlights the superior efficiency of our approach in handling gradient calculations for multilevel optimization problems.

\subsection{Convergence Analysis}

\subsubsection{Convergence for Multilevel Optimization}
When analyzing the convergence of gradient-based optimization algorithm for multilevel problems, we find it difficult to estimate the difference between the $f_1$ value in step $N$ and the optimal value. To address this, we shifted our focus to the analysis of the average derivatives. Specifically, we aimed to demonstrate that the gradient $\frac{\dif f_1}{\dif \mathbf x_1}$ diminishes progressively throughout the updating process of $\mathbf x_1$. Consequently, we present the following theorem:
\begin{theorem}\label{new_triconverge}
    Assume that $f_i$ for multilevel optimization is continuous with $i$-th order derivatives. If the second-order derivative of $f_1(\mathbf x_1, \mathbf x_2, \cdots, \mathbf x_n)$ with respect to $\mathbf x_1$, i.e. the Hessian matrix $\frac{\dif^2 f_1}{\dif \mathbf x_1^2}$, is positive definite and the maximum eigenvalue over every $\mathbf x_1$ is bounded, then we have

    \[\mathbb E \big(\dfrac{\dif f_1}{\dif \mathbf x_1}\big)^2 \leq \mathcal O\big(\frac{1}{N}\big)\]

\end{theorem}
Detailed proof can be seen in Appendix \ref{new_triconverge_proof}. In this way, it becomes evident that with an increasing number of update rounds, the expected gradient consistently diminishes. This decline serves as an indication of the algorithm's convergence.
\subsubsection{Convergence for General Cases}
When the domain $S$ of $\mathbf x_1$ is a compact convex set, according to the theorem \ref{new_triconverge}, gradient-based optimization algorithm converges to the optimal value. We point out that in general, the algorithm also converges. For simplicity, we present the results in the case of scalar form.
\begin{theorem}\label{nconverge}
    Assuming that $f_i$ are analytic functions satisfying
    \begin{enumerate}
    \item 
    $|\frac{\partial f_1}{\partial x_i}|\leq N$ and $|\frac{\partial^2 f_1}{\partial x_i^2}|\leq M$ for all $x_i$.
    \item 
    $|\frac{\partial x_i}{\partial x_j}|\leq K$ and $|\frac{\partial^2 x_i}{\partial x_j^2}|\leq S$ for all $j < i$.
    \end{enumerate}
Then we have
\[\mathbb E (\dfrac{\dif f_1}{\dif x_1})^2\leq \dfrac{F_0(\beta - \frac{C\beta^2n^nJ^{n+1}}{2})^{-1}}{N} \sim \mathcal O\big(\frac{n^nJ^{n+1}}{N}\big)\]
where $C$ is a constant and $J$ is $\max\{M,N,K,S,1\}$, a constant related to the Lipschitz constant of various derivatives of $f_i$.
\end{theorem}
Detailed proof of theorem \ref{nconverge} can be seen in Appendix \ref{a4}.




\section{Experiments}\label{sec:exp}

In the previous section, we detailed the process for computing the full gradient concerning the first input. However, it remains crucial to evaluate whether this approach of computing the full gradient offers advantages over alternative methodologies. To this end, in the current section, we apply our methods to a meticulously designed experiment and further conduct hyperparameter optimization, following the guidelines and procedures outlined in \citet{sato2021gradient}. This comparative analysis aims to demonstrate the efficacy and efficiency of our gradient computation method within practical application scenarios.

\subsection{Experimental Design}

\paragraph{Generalization of Stackelberg's Model} The Stackelberg model, established in the work of \citet{stackelberg1952theory}, describes a hierarchical oligopoly framework where a dominant firm, referred to as the leader, first decides its output or pricing strategy. Subsequently, the remaining firms, labeled as followers, adjust their strategies accordingly, fully aware of the leader's decisions. We extend this model to a trilevel hierarchy by introducing a secondary leader who is fully informed of the primary leader's actions. The outputs of the first leader, second leader, and follower are represented by \(\mathbf x\), \(\mathbf y\), and \(\mathbf z\), respectively. We consider a demand curve defined by \(P = \mathbf{1 - x - y - z}\), thereby incorporating the interactive dynamics between multiple leaders and a follower within the model's strategic framework.

The loss functions are defined as:
\begin{align*}
f_1(x,y,z) &= -\mathbf{x^\top (1-x-y-z)}, \\
f_2(x,y,z) &= -\mathbf{y^\top (1-x-y-z)}, \\
f_3(x,y,z) &= -\mathbf{z^\top (1-x-y-z)}.
\end{align*}

The optimal solution for this generalized model can be analytically determined as \(\mathbf x = \frac{\mathbf{1}}{2}\).

\paragraph{Hyperparameter optimization} In this experiment, we follow \citet{sato2021gradient} and consider an adversarial scenario where a learner optimizes the hyperparameter $\lambda$ to derive a noise-robust model, while an attacker poisons the training data to make it inaccurate. The model is formulated as follows:
\begin{align*}
f_1(\lambda, P,\theta) &= \frac{1}{m} ||y_{val} - X_{val}\theta||_2^2, \\
f_2(\lambda, P,\theta) &= -\frac{1}{n} ||y_{train} \! - \! (\! X_{train} \! + \! P)\theta||_2^2 + \frac{c}{nd}||P'||_2^2, \\
f_3(\lambda, P,\theta) &= \frac{1}{n} ||y_{train} \! - \! ( \! X_{train} \! + \! P)\theta||_2^2 \! + \! \exp (\lambda) \frac{||\theta||_{1^*}}{d}
\end{align*}

Where \(X_{train}\) and \(X_{val}\) represent the feature sets of the training and validation datasets, respectively, and \(y_{train}\) and \(y_{val}\) denote the corresponding target values for these datasets. The symbols \(m\) and \(n\) indicate the sizes of the validation and training datasets, respectively, while \(d\) represents the dimensionality of the features, and \(c\) signifies the penalty imposed on the attacker. For a
comprehensive understanding, we recommend readers refer to \citet{sato2021gradient}. Following the methodology outlined in \citet{sato2021gradient}, we set \(m=100\) for the size of the validation dataset, \(n=40\) for the size of the training dataset, and \(c=100\) as the penalty for the attacker.

\paragraph{Experiment Setting} In our experimental setup, we approach the solution of the low-level optimization problem following \citet{sato2021gradient}. Specifically, this entails performing \(30\) updates on \(\bf y\) for every update on \(\bf x\), and \(3\) updates on \(\bf z\) for every update on \(\bf y\). We employ the Adam optimizer \citep{kingma2017adam} for \(\mathbf x\) with \(\beta_1 = 0.5\), \(\beta_2 = 0.999\). For faster convergence, we set the learning rate at step $t$ to $0.1\times 0.99^t$. For \(\bf y\) and \(\bf z\), we utilize standard gradient descent, setting all learning rates to \(10^{-2}\). Specifically in our method, we use $3$ iterations of conjugate gradient method \citep{hestenes1952methods} to get rid of the matrix inverse in line 9 of \cref{alg}.

To assess the efficacy of our approach, we draw comparisons with Vanilla Gradient Descent (VGD) (which involves taking only \(\bf x\)'s partial derivative), Finite Difference (FD) \citep{liu2018darts}, and ITD \citep{sato2021gradient}. For the first experiment, which has an analytical solution, we measure performance based on the mean square error (MSE) against the ground truth. In the second experiment, we further conduct an inference run, optimizing \(\bf y\) and \(\bf z\) until convergence, and subsequently report on the numerical evaluation of \(f_1(\mathbf{x}, \mathbf{y}^*, \mathbf{z}^*)\). All the experiments are run on a single A100 GPU.

\subsection{Experimental Results}

\paragraph{Generalized Stackelberg's model}. 

We test all four methods on this model. The result is shown in \cref{compare}, where different curves correspond to different methods with $x$-axis representing the optimization step of $x$, and the $y$-axis representing the MSE to ground truth. The empirical findings verify the effectiveness of our proposed method, and clearly demonstrate the advantage compared to other alternatives.

\begin{figure}[ht]
    \centering
    \includegraphics[width=0.45\linewidth]{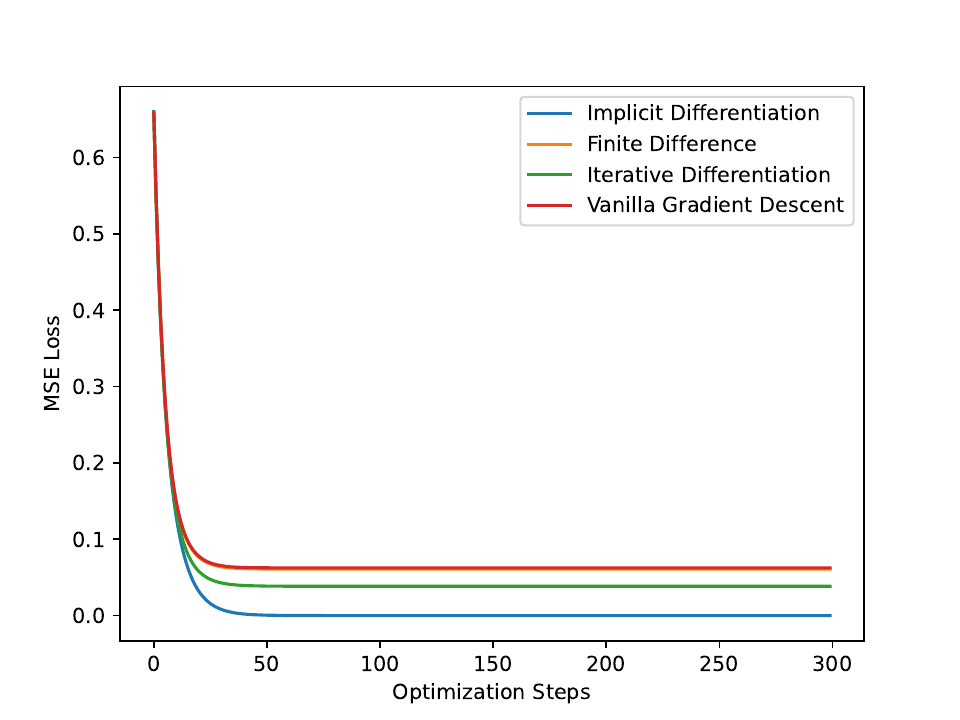}
    \caption{The MSE in Generalized Stackelberg's model. Our method converges significantly faster than all the alternatives, and is the only method that do not fall into local minima.}
    \label{compare}
\end{figure}

\begin{figure}[ht]
    \centering
    \begin{subfigure}[b]{0.4\textwidth}
         \centering
         \includegraphics[width=\textwidth]{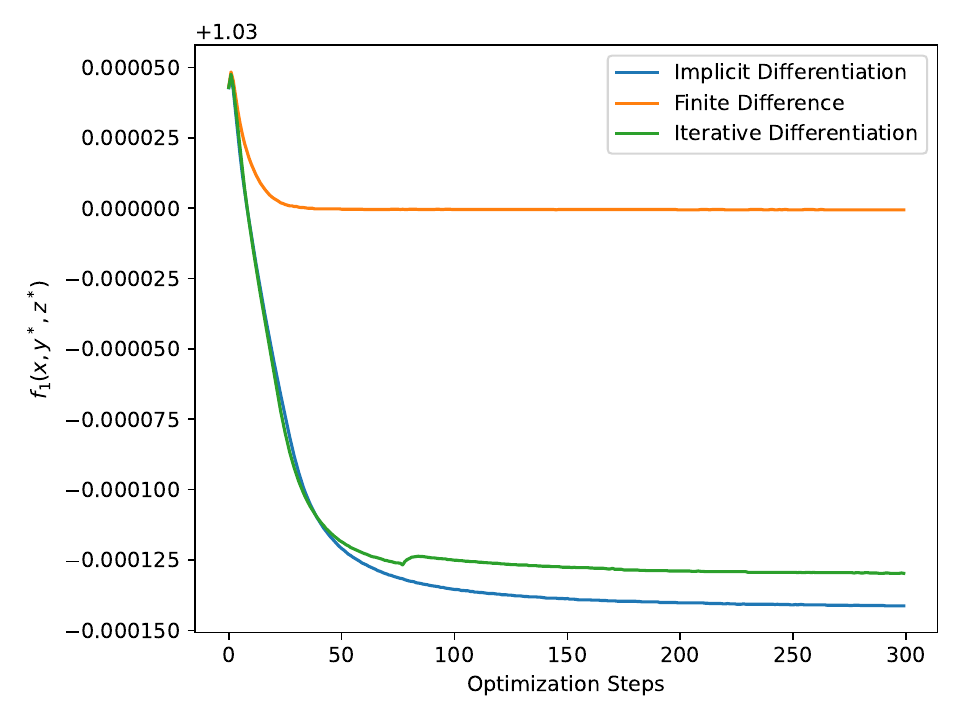}
         \caption{Red wine quality dataset}
     \end{subfigure} 
     \begin{subfigure}[b]{0.4\textwidth}
         \centering
         \includegraphics[width=\textwidth]{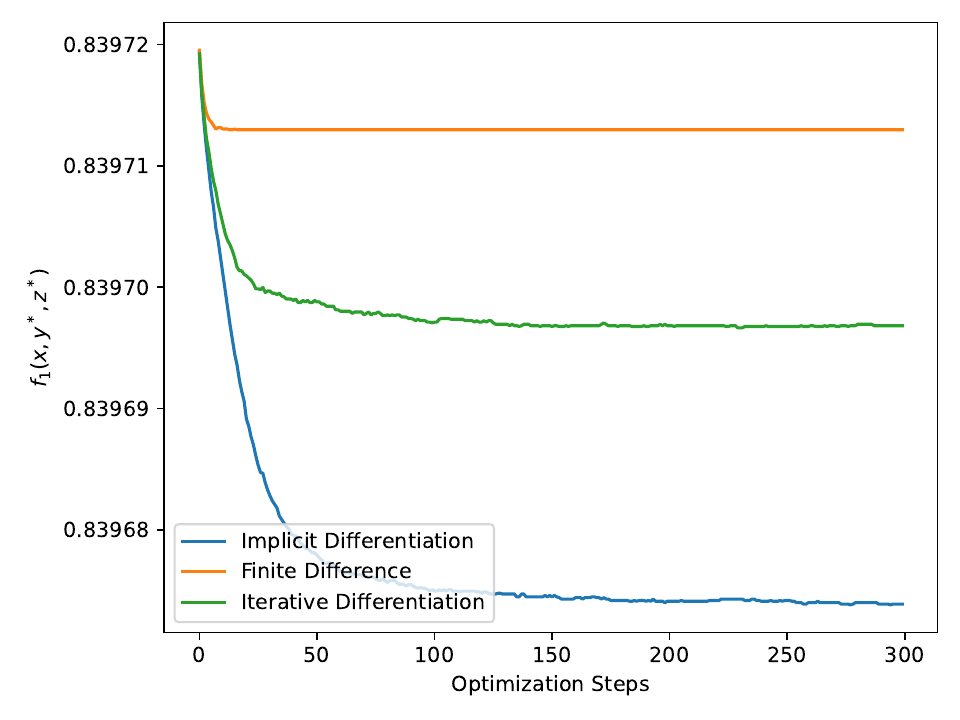}
         \caption{White wine quality dataset}
     \end{subfigure}
    \caption{In the task of hyperparameter optimization, both the FD method and the ITD approach tend to converge to local minima.}
    \label{hyper}
\end{figure}


\paragraph{Hyperparameter optimization}

\begin{table}[ht]
\caption{Ratio of computational time relative to the vanilla method in Hyperparameter Optimization. This clearly demonstrates the efficiency of our method compared to \citet{sato2021gradient}.}
\begin{center}
\begin{small}
\begin{sc}
\begin{tabular}{|c|c|c|c|c|}
\hline
    Algorithm & VGD & FD & ITD  & ID (ours) \\ 
\hline
    Time per update & 1 & 2.0 & 10.3  & 3.1 \\ 
\hline
\end{tabular}
\end{sc}
\end{small}
\end{center}
\label{tab:t}
\end{table}

We evaluated these methods on regression tasks using the red and white wine quality datasets, as described in \citet{cortez2009modeling}, adopting the approach outlined in \citet{sato2021gradient}. Each feature and target within the datasets was normalized prior to analysis. The outcomes of this evaluation are presented in \cref{hyper}. We omitted the VGD method from our testing because the loss function \(f_1\) does not explicitly include \(\lambda\). Unlike other optimization algorithms that tend to get trapped in local minima, our method exhibits consistent improvement across iterations, surpassing the performance of the baseline methods. Additionally, the ratio of computational time relative to the vanilla method is presented in \cref{tab:t}.

\section{Related Work}
A majority of work in this field tackle on the simplified bilevel cases. Our work draw inspiration from \citet{gould2016differentiating}, which belongs to the school of first-order methods. Since then, more modern techniques have been proposed with stochastic nature as momentum recursive \citep{cutkosky2019momentum, tran2019hybrid} or variance reduction \citep{nguyen2017sarah, li2018simple} for optimizing computational complexity. Inspired by the technique leveraged in solving bilevel optimization problems, a wide array of machine learning applications flourish, such as hyperparameter optimization \citep{bennett2006model, franceschi2017forward, lorraine2020optimizing}, reinforcement learning \citep{konda1999actor, rajeswaran2020game, hong2023two}, neural architecture search \citep{liu2018darts, zhang2021idarts}, meta-learning \citep{finn2017model, franceschi2018bilevel, rajeswaran2019meta}, and so on. Yet these works still leave a lot of room for improvement. For example, \citet{liu2018darts} leverages finite difference for neural architecture search, which is in essence a zero-order optimization method not utilizing the information to its full.

Recently more attempts have been made towards trilevel or multilevel problems \citep{tilahun2012new, shafiei2021trilevel}. Inspired by the success of iterative differentiation in bilevel optimization, \citet{sato2021gradient} proposed an approximate gradient-based method with theoretical guarantee. Nevertheless, the convergence analysis only holds asymptotically, hence requiring a considerably large iteration number in lower level, making it unsuitable for practical applications. There have also been efforts on the engineering side \citep{grefenstette2019generalized, blondel2022efficient}. \citet{blondel2022efficient} proposed a modular framework for implicit differentiation, but failed to support multilevel optimization due to override of JAX's automatic differentiation, which integrates the chain rule. \citet{choe2022betty} developed a software library for efficient automatic differentiation of multilevel optimization, but differs from our focus in that we propose the method for composing best-response Jacobian, which serves as the premise of their work.

\section{Conclusion}
\paragraph{Summary}
In this study, we introduce an automatic differentiation method tailored for the gradient computation of trilevel optimization problems. We also develope a recursive algorithm designed for differentiation in general multilevel optimization scenarios, providing theoretical insights into the convergence rates applicable to both trilevel configurations and broader \(n\)-level contexts. Through comprehensive testing, our approach has been demonstrated to outperform existing methods significantly.

\paragraph{Limitations and Future Directions}
The effectiveness of our algorithm is most evident within simpler systems and when handling a smaller number of levels \(n\). While it yields relatively precise gradients for smooth and convex objective functions, challenges arise in optimizing non-convex functions, potentially leading to non-invertible Hessian matrices. Moreover, the efficiency of our gradient computation diminishes as the number of levels \(n\) and the dimensionality of the features \(d\) increase, due to the computational complexity scaling at \(\mathcal{O}(d^3n^4)\). Identifying heuristic methods to approximate gradients more efficiently remains for future exploration.

\bibliographystyle{abbrvnat}
\bibliography{neurips_2024}

\newpage
\appendix

\section{Theoretical Proofs}


\textbf{The proof of lemma \ref{lemma_31}:} \label{2.1}
Let $f:\mathbb R^{m+n}\to \mathbb R$ be a continuous function with second derivatives and $\det (\frac{\partial ^2 f}{\partial \bf{y} ^2} ) \neq 0$. Let $\bf{g}(\bf{x}) = \argmin _{\bf y} f(\bf x, \bf y)$. Then the derivative of $\bf g$ with respect to $\bf x$ is
$$\dfrac{\dif \mathbf{g}(\mathbf{x})}{\dif \mathbf{x}} = - f_{YY}(\mathbf{x}, \mathbf{g}(\mathbf{x}))^{-1} f_{XY}(\mathbf{x}, \mathbf{g}(\mathbf{x}))$$

\begin{proof}
\label{2.2}
First, since \( \mathbf{g}(\mathbf{x}) \) minimizes \( f \) with respect to \( \mathbf{y} \), the first order condition gives:
\[ \frac{\partial f}{\partial \mathbf{y}}(\mathbf{x}, \mathbf{g}(\mathbf{x})) = 0. \]

Differentiating this condition with respect to \( \mathbf{x} \) and applying the chain rule, we obtain:
\[ \frac{\mathrm{d}}{\mathrm{d} \mathbf{x}} \left(\frac{\partial f}{\partial \mathbf{y}}(\mathbf{x}, \mathbf{g}(\mathbf{x}))\right) = 0. \]

This differentiation can be expanded as:
\[ \frac{\partial^2 f}{\partial \mathbf{x} \partial \mathbf{y}} + \frac{\partial^2 f}{\partial \mathbf{y}^2} \frac{\mathrm{d} \mathbf{g}}{\mathrm{d} \mathbf{x}} = 0. \]

Here, \( \frac{\partial^2 f}{\partial \mathbf{x} \partial \mathbf{y}} \) represents the matrix of mixed second partial derivatives of \( f \) with respect to \( \mathbf{x} \) and \( \mathbf{y} \), and \( \frac{\partial^2 f}{\partial \mathbf{y}^2} \) is the Hessian matrix of \( f \) with respect to \( \mathbf{y} \).

To solve for \( \frac{\mathrm{d} \mathbf{g}}{\mathrm{d} \mathbf{x}} \), rearrange the above equation:
\[ \frac{\partial^2 f}{\partial \mathbf{y}^2} \frac{\mathrm{d} \mathbf{g}}{\mathrm{d} \mathbf{x}} = -\frac{\partial^2 f}{\partial \mathbf{x} \partial \mathbf{y}}. \]

Assuming \( \frac{\partial ^2 f}{\partial \mathbf{y}^2} \) is invertible (as implied by $ \det \left(\frac{\partial ^2 f}{\partial \mathbf{y}^2}\right) \neq 0 $), we multiply both sides by the inverse of this matrix:
$$ \frac{\mathrm{d} \mathbf{g}}{\mathrm{d} \mathbf{x}} = -\left(\frac{\partial ^2 f}{\partial \mathbf{y}^2}\right)^{-1} \frac{\partial ^2 f}{\partial \mathbf{x} \partial \mathbf{y}} $$

This completes the proof.
\end{proof}

\textbf{The proof of lemma \ref{lemma1}:}

Let $f: \mathbb{R}^{d_1+d_2+d_3} \to \mathbb{R}$ be a continuous function with second derivatives. Let $\mathbf{g(x,y)} = \argmin_{\mathbf{z} \in \mathbb{R}^{d_3}}f(\mathbf{x,y,z})$, then the following properties about $g$ holds:\\
(a)
$$\dfrac{\partial \mathbf{g}}{\partial \mathbf{x}} = -f_{ZZ}^{-1}f_{XZ} $$
(b)
$$\dfrac{\partial \mathbf{g}}{\partial \mathbf{y}} = -f_{ZZ}^{-1}f_{YZ} $$

\begin{proof}\label{3.1}
To find \( \frac{\partial \mathbf{g}}{\partial \mathbf{x}} \), differentiate the first order condition with respect to \( \mathbf{x} \) using the chain rule:
\[ \frac{\mathrm{d}}{\mathrm{d} \mathbf{x}} \left(\frac{\partial f}{\partial \mathbf{z}}(\mathbf{x, y, g(x, y)})\right) = 0. \]

Expanding this using the chain rule yields:
\[ \frac{\partial^2 f}{\partial \mathbf{x} \partial \mathbf{z}} + \frac{\partial^2 f}{\partial \mathbf{z}^2} \frac{\partial \mathbf{g}}{\partial \mathbf{x}} = 0. \]

Solving for \( \frac{\partial \mathbf{g}}{\partial \mathbf{x}} \):
\[ \frac{\partial \mathbf{g}}{\partial \mathbf{x}} = -\left(\frac{\partial ^2 f}{\partial \mathbf{z} ^2}\right)^{-1} \frac{\partial ^2 f}{\partial \mathbf{x} \partial \mathbf{z}}. \]

The proof for part (b) is similar.
\end{proof}

\textbf{The proof of lemma \ref{lemma2}:}

Let $f: \mathbb{R}^{d_1+d_2+d_3} \to \mathbb{R}$ and $\mathbf{g}: \mathbb{R}^{d_1+d_2} \to \mathbb{R}^{d_3}$ be continuous functions with second derivatives. Let $\mathbf{h(x)} = \argmin_{y\in \mathbb{R}}f(\mathbf{x,y,g(x,y)})$, then the derivative of $\bf h$ with respect to $\bf x$ is:
\begin{multline*}
    \dfrac{\partial \mathbf{h}}{\partial \mathbf{x}} = -\left(f_{YY} + 2f_{ZY} \dfrac{\partial \mathbf{g}}{\partial \mathbf{y}}  + (\dfrac{\partial \mathbf{g}}{\partial \mathbf{y}})^\top f_{ZZ} \dfrac{\partial \mathbf{g}}{\partial \mathbf{y}} + \sum _{i=1}^{d_3} (f_Z)_i \dfrac{\partial}{\partial \mathbf{y}} (\dfrac{\partial \mathbf{g}}{\partial \mathbf{y}})_i\right)^{-1}\\
    \left(f_{XY} + f_{ZY} \dfrac{\partial\mathbf{g}}{\partial \mathbf x}  + (\dfrac{\partial \mathbf g}{\partial \mathbf y})^\top f_{XZ} +  (\dfrac{\partial \mathbf g}{\partial \mathbf y})^\top f_{ZZ}\dfrac{\partial \mathbf g}{\partial \mathbf x} + \sum_{i=1}^{d_3} (f_Z)_i \dfrac{\partial}{\partial \mathbf x}(\dfrac{\partial \mathbf g}{\partial \mathbf y})_i\right )
\end{multline*}

\begin{proof}\label{3.2}
First, we establish the condition for \( \mathbf{h}(\mathbf{x}) \) by setting the gradient of \( f \) with respect to \( \mathbf{y} \) equal to zero, considering the implicit function defined by \( \mathbf{g} \). The condition \( \frac{\partial f}{\partial \mathbf{y}} + \frac{\partial f}{\partial \mathbf{z}} \frac{\partial \mathbf{g}}{\partial \mathbf{y}} = 0 \) must hold at the minimizer \( \mathbf{y} = \mathbf{h}(\mathbf{x}) \).

Differentiating this condition with respect to \( \mathbf{x} \) gives:
\[
\frac{\mathrm{d}}{\mathrm{d} \mathbf{x}} \left( \frac{\partial f}{\partial \mathbf{y}} + \frac{\partial f}{\partial \mathbf{z}} \frac{\partial \mathbf{g}}{\partial \mathbf{y}} \right) = 0. \]

Expanding this derivative using the product and chain rules, we obtain:
\begin{multline*}
f_{XY} + f_{YY} \frac{\partial \mathbf{h}}{\partial \mathbf{x}} + f_{YZ} \left( \frac{\partial \mathbf{g}}{\partial \mathbf{x}} + \frac{\partial \mathbf{g}}{\partial \mathbf{y}} \frac{\partial \mathbf{h}}{\partial \mathbf{x}} \right) +\\
\frac{\partial \mathbf{g}}{\partial \mathbf{y}} \left( f_{XY} + f_{YZ} \frac{\partial \mathbf{h}}{\partial \mathbf{x}} + f_{ZZ} \left( \frac{\partial \mathbf{g}}{\partial \mathbf{x}} + \frac{\partial \mathbf{g}}{\partial \mathbf{y}} \frac{\partial \mathbf{h}}{\partial \mathbf{x}} \right) \right) + f_Z \left( \frac{\partial^2 \mathbf{g}}{\partial \mathbf{x} \partial \mathbf{y}} + \frac{\partial^2 \mathbf{g}}{\partial \mathbf{y}^2} \frac{\partial \mathbf{h}}{\partial \mathbf{x}} \right) = 0.
\end{multline*}

Rearranging terms and collecting like terms involving \( \frac{\partial \mathbf{h}}{\partial \mathbf{x}} \), we find:
\begin{multline*}
\left(f_{YY} + 2f_{ZY} \dfrac{\partial \mathbf{g}}{\partial \mathbf{y}}  + (\dfrac{\partial \mathbf{g}}{\partial \mathbf{y}})^\top f_{ZZ} \dfrac{\partial \mathbf{g}}{\partial \mathbf{y}} + \sum _{i=1}^{d_3} (f_Z)_i \dfrac{\partial}{\partial \mathbf{y}} (\dfrac{\partial \mathbf{g}}{\partial \mathbf{y}})_i\right) \frac{\partial \mathbf{h}}{\partial \mathbf{x}}\\
= -\left(f_{XY} + f_{ZY} \dfrac{\partial\mathbf{g}}{\partial \mathbf x}  + (\dfrac{\partial \mathbf g}{\partial \mathbf y})^\top f_{XZ} +  (\dfrac{\partial \mathbf g}{\partial \mathbf y})^\top f_{ZZ}\dfrac{\partial \mathbf g}{\partial \mathbf x} + \sum_{i=1}^{d_3} (f_Z)_i \dfrac{\partial}{\partial \mathbf x}(\dfrac{\partial \mathbf g}{\partial \mathbf y})_i\right )
\end{multline*}

Inverting the matrix on the left-hand side to solve for \( \frac{\partial \mathbf{h}}{\partial \mathbf{x}} \), we obtain:

\begin{multline*}
    \dfrac{\partial \mathbf{h}}{\partial \mathbf{x}} = -\left(f_{YY} + 2f_{ZY} \dfrac{\partial \mathbf{g}}{\partial \mathbf{y}}  + (\dfrac{\partial \mathbf{g}}{\partial \mathbf{y}})^\top f_{ZZ} \dfrac{\partial \mathbf{g}}{\partial \mathbf{y}} + \sum _{i=1}^{d_3} (f_Z)_i \dfrac{\partial}{\partial \mathbf{y}} (\dfrac{\partial \mathbf{g}}{\partial \mathbf{y}})_i\right)^{-1}\\
    \left(f_{XY} + f_{ZY} \dfrac{\partial\mathbf{g}}{\partial \mathbf x}  + (\dfrac{\partial \mathbf g}{\partial \mathbf y})^\top f_{XZ} +  (\dfrac{\partial \mathbf g}{\partial \mathbf y})^\top f_{ZZ}\dfrac{\partial \mathbf g}{\partial \mathbf x} + \sum_{i=1}^{d_3} (f_Z)_i \dfrac{\partial}{\partial \mathbf x}(\dfrac{\partial \mathbf g}{\partial \mathbf y})_i\right )
\end{multline*}

This expression for \( \frac{\partial \mathbf{h}}{\partial \mathbf{x}} \) effectively describes how the minimizer \( \mathbf{h}(\mathbf{x}) \) changes as a function of \( \mathbf{x} \). The matrix inversion reflects the dependency of the minimization outcome on the curvature of \( f \) with respect to \( \mathbf{y} \) and the interactions between \( f \) and \( \mathbf{g} \) via their cross-derivatives. The derivative \( \frac{\partial \mathbf{h}}{\partial \mathbf{x}} \) encapsulates the net effect of these interactions and the underlying geometry of the function \( f \) as modified by the mapping \( \mathbf{g} \).
\end{proof}

\textbf{The proof of theorem \ref{new_triconverge}:}
Assume that $f_i$ for multilevel optimization is continuous with $i$-th order derivatives. If the second-order derivative of $f_1(\mathbf x_1, \mathbf x_2, \cdots, \mathbf x_n)$ with respect to $\mathbf x_1$, i.e. the Hessian matrix $\frac{\dif^2 f_1}{\dif \mathbf x_1^2}$, is positive definite and the maximum eigenvalue over every $\mathbf x_1$ is bounded, then we have

    \[\mathbb E \big(\dfrac{\dif f_1}{\dif \mathbf x_1}\big)^2 \leq \mathcal O\big(\frac{1}{N}\big)\]
\begin{proof}\label{new_triconverge_proof}
The basic proof approach is to compute the difference of $f_1$ value between two steps using Lagrange Mean Value Theorem and note that $f_1$ is actually a function of $\mathbf x_1$ because other variables $\mathbf x_2, \cdots, \mathbf x_n$ are actually functions of $\mathbf x_1$. For simplicity, we use $\mathbf x^i$ to represent $\mathbf x^i(\mathbf x_1^i)=(\mathbf x_1^i, \mathbf x_2^i(\mathbf x_1^i), \cdots, \mathbf x_n^i(\mathbf x_1^i,\cdots, \mathbf x_{n-1}^i))$ , where $\mathbf x_2^i,\cdots, \mathbf x_n^i$ are optimal value as discussed in the $i$-th round. Then the difference will be
    \begin{align*}
        &f_1(\mathbf x^{i+1})- f_1(\mathbf x^{i})\\
        =&f_1\big(\mathbf x_1^{i+1},\cdots,\mathbf x_n^{i+1}(\mathbf x_1^{i+1},\cdots, \mathbf x_{n-1}^{i+1})\big) - f_1\big(\mathbf x_1^{i},\cdots,\mathbf x_n^{i}(\mathbf x_1^{i},\cdots, \mathbf x_{n-1}^{i})\big)\\
        =& \frac{\dif f_1}{\dif \mathbf x_1}(\mathbf x^i)\cdot(\mathbf x_1^{i+1} - \mathbf x_1^i) + \frac{1}{2}(\mathbf x_1^{i+1} - \mathbf x_1^i)^T\frac{\dif^2 f_1}{\dif \mathbf x_1^2}(\mathbf \xi)(\mathbf x_1^{i+1} - \mathbf x_1^i)\\
    \end{align*}
where $\mathbf \xi$ is some value between $\mathbf x^i$ and $\mathbf x^{i+1}$ according to Lagrange Mean Value Theorem. 

The Hessian matrix $\dfrac{\dif^2 f_1}{\dif \mathbf x_1^2}$ is positive definite, so we have
\begin{align*}
    (\mathbf x_1^{i+1} - \mathbf x_1^i)^T\frac{\dif^2 f_1}{\dif \mathbf x_1^2}(\mathbf \xi)(\mathbf x_1^{i+1} - \mathbf x_1^i) &\leq \lambda_{max} (\mathbf x_1^{i+1} - \mathbf x_1^i)^T(\mathbf x_1^{i+1} - \mathbf x_1^i)\\
    &=\lambda_{max} (\mathbf x_1^{i+1} - \mathbf x_1^i)^2
\end{align*}
where $\lambda_{max}$ is a bound of maximum eigenvalue over every $\mathbf x_1$ of the matrix $\dfrac{\dif^2 f_1}{\dif \mathbf x_1^2}$, i.e.
\[\lambda\big(\dfrac{\dif^2 f_1}{\dif \mathbf x_1^2}(\mathbf x(\mathbf x_1, \cdots, \mathbf x_n))\big) \leq \lambda_{max}\ \forall \mathbf x_1\]
Note that we update $\mathbf x_1$ by 
\[\mathbf x_1 \leftarrow \mathbf x_1 - \beta\frac{\dif f_1}{\dif \mathbf x_1}(\mathbf x_1,\mathbf x_2, \cdots, \mathbf x_n)\]
we can replace $\mathbf x^{i+1} - \mathbf x^i$:
\begin{align*}
        &f_1(\mathbf x^{i+1})- f_1(\mathbf x^{i})\\
        =& \frac{\dif f_1}{\dif \mathbf x_1}(\mathbf x^i)\cdot(\mathbf x_1^{i+1} - \mathbf x_1^i) + \frac{1}{2}(\mathbf x_1^{i+1} - \mathbf x_1^i)^T\frac{\dif^2 f_1}{\dif \mathbf x_1^2}(\mathbf \xi)(\mathbf x_1^{i+1} - \mathbf x_1^i)\\
        \leq &\frac{\dif f_1}{\dif \mathbf x_1}(\mathbf x^i)\cdot(\mathbf x_1^{i+1} - \mathbf x_1^i) + \frac{1}{2}\lambda_{max}(\mathbf x_1^{i+1} - \mathbf x_1^i)^2\\
        =&-\beta \big(\frac{\dif f_1}{\dif \mathbf x_1}(\mathbf x^i)\big)^2 + \frac{1}{2}\lambda_{max}\beta^2 \big(\frac{\dif f_1}{\dif \mathbf x_1}(\mathbf x^i)\big)^2\\
        =&\big(\frac{\lambda_{max}\beta^2}{2} - \beta\big)\big(\frac{\dif f_1}{\dif \mathbf x_1}(\mathbf x^i)\big)^2
    \end{align*}
So we get
\begin{align*}
    f_1(\mathbf x^{i+1})- f_1(\mathbf x^{i})
        &=f_1\big(\mathbf x_1^{i+1},\cdots,\mathbf x_n^{i+1}(\mathbf x_1^{i+1},\cdots, \mathbf x_{n-1}^{i+1})\big) - f_1\big(\mathbf x_1^{i},\cdots,\mathbf x_n^{i}(\mathbf x_1^{i},\cdots, \mathbf x_{n-1}^{i})\big)\\
        &\leq \big(\frac{\lambda_{max}\beta^2}{2} - \beta\big)\big(\frac{\dif f_1}{\dif \mathbf x_1}(\mathbf x^i)\big)^2
\end{align*}
After summing both sides, we find that the sum of the squared derivatives for each step can be bounded by a constant which is related to the initial point we select:
\begin{align*}
     &\big(\beta - \frac{\beta^2\lambda_{max}}{2}\big)\sum_{i=0}^{N-1}\big(\frac{\dif f_1}{\dif \mathbf x_1}(\mathbf x^i)\big)^2\\
        \leq& f_1(\mathbf x^0) - f_1(\mathbf x^N)\\
        \leq& f_1(\mathbf x^0)
\end{align*}
So when $\beta\leq\frac{2}{\lambda_{max}}$, the expectation of the squared derivative can be approximated as
\[\mathbb E (\dfrac{\dif f_1}{\dif \mathbf x_1})^2 \sim \dfrac{1}{N}\sum_{i=0}^{N-1}\big(\dfrac{\dif f_1}{\dif \mathbf x_1}\big(\mathbf x^i)\big)^2 \leq \dfrac{1}{N}\big(\beta - \frac{\beta^2\lambda_{max}}{2}\big)^{-1}f_1(\mathbf x_0) = \mathcal O(\frac{1}{N})\]
by Central Limit Theorem.
\end{proof}
\textbf{The proof of theorem \ref{nconverge}:}

    Assume that $f_i$ are analytic functions satisfying
    \begin{enumerate}
    \item 
    $|\frac{\partial f_1}{\partial x_i}|\leq N$ for all $x_i$.
    \item 
$|\frac{\partial^2 f_1}{\partial x_i^2}|\leq M$ for all $x_i$.
    \item 
$|\frac{\partial x_i}{\partial x_j}|\leq K$ for all $j < i$.
    \item 
    $|\frac{\partial^2 x_i}{\partial x_j^2}|\leq S$ for all $j < i$.
\end{enumerate}
we have
\[\mathbb E (\dfrac{\dif f_1}{\dif x_1})^2\leq \dfrac{F_0(\beta - \frac{C\beta^2n^nJ^{n+1}}{2})^{-1}}{N} \sim \mathcal O\big(\frac{n^nJ^{n+1}}{N}\big)\]
Here $F_0$ is the value of $f_1$ at the initial point, $J$ is $\max\{M,N,K,S,1\}$, $\beta$ is the step size when we update $x_1$, $n$ is the number of layers, $N$ is the running rounds, and $C$ is a constant. And $x_2, x_3, \cdots, x_n$ are determined by the following system of equations:
$$\begin{cases}\label{equations}
    &x_n^\star  = x_n(x_1,\cdots, x_{n-1}) = \arg\min_{x_n}f_n(x_1,\cdots, x_n)\\
    &x_{n-1}^\star = x_{n-1}(x_1,\cdots, x_{n-2}) = \arg\min_{x_{n-1}}f_{n-1}(x_1,\cdots, x_n^\star)\\
    &\vdots\\
    &x_3^\star = x_3(x_1, x_2) = \arg\min_{x_3}f_3(x_1,\cdots, x_n^*)\\
    &x_2^\star = x_2(x_1) = \arg\min_{x_2}f_2(x_1, x_2, x_3^\star\cdots, x_n^\star)
\end{cases}$$
The $\frac{\partial x_i}{\partial x_j}$ and $\frac{\partial^2 x_i}{\partial x_j^2}$ can be obtained like what we do in lemma \ref{lemma1} and lemma \ref{lemma2}.

We begin with several lemmas:
\begin{lemma}\label{dxidx1}
    When assumption 3 holds, we have
    \[|\dfrac{\dif x_i}{\dif x_1}(x)|\leq \mathcal O(((i-1)J)^{i-1})\]
    for any $x$.
\end{lemma}
\begin{proof}\label{a1}
Firstly, we have
\[|\dfrac{\dif x_2}{\dif x_1}(x)| = |\dfrac{\partial x_2}{\partial x_1}(x)|\leq K =\mathcal O(J)\]
Assume that 
\[|\dfrac{\dif x_n}{\dif x_1}(x)|\leq \mathcal O(((n-1)J)^{n-1})\]
Then
\begin{align*}
    |\dfrac{\dif x_{n+1}}{\dif x_1}(x)| &= |\sum\limits_{j < n+1}\frac{\partial x_{n+1}}{\partial x_j}\cdot \frac{\dif x_j}{\dif x_1}(x)|\\
    &\leq \sum\limits_{j < n+1}|\frac{\partial x_{n+1}}{\partial x_j}(x)|\cdot |\frac{\dif x_j}{\dif x_1}(x)|\\
    &\leq \sum\limits_{j < n+1} K \mathcal O(((j-1)J)^{j-1})\\
    &\leq K n \mathcal O(((n-1)J)^{n-1})\\
    &= \mathcal O((nJ)^{n})
\end{align*}
By mathematical induction, we can get the conclusion
\[|\dfrac{\dif x_i}{\dif x_1}(x)|\leq \mathcal O(((i-1)J)^{i-1})\]
\end{proof}
\begin{lemma}\label{difdxidx1}
    When assumption 3 and 4 hold, we have
    \[|\dfrac{\dif x_i}{\dif x_1}(y) - \dfrac{\dif x_i}{\dif x_1}(z)|\leq \mathcal O(((i-1)J)^{i-1}S)|y-z|\]
    for any $y$ and $z$.
\end{lemma}
\begin{proof}\label{a2}
    Firstly, we can get
    \[|\dfrac{\dif x_2}{\dif x_1}(y) - \dfrac{\dif x_2}{\dif x_1}(z)|\leq S|y-z| = \mathcal O(JS)|y-z|\]
    easily from the assumption 4.
    
    Assume that
    \[|\dfrac{\dif x_n}{\dif x_1}(y) - \dfrac{\dif x_n}{\dif x_1}(z)|\leq \mathcal O(((n-1)J)^{n-1}S)|y-z|\]
    Then
\begin{align*}
    |\dfrac{\dif x_{n+1}}{\dif x_1}(y) - \dfrac{\dif x_{n+1}}{\dif x_1}(z)| &= |\sum\limits_{j < n+1}\frac{\partial x_{n+1}}{\partial x_j}\cdot \frac{\dif x_j}{\dif x_1}(y) - \sum\limits_{j < n+1}\frac{\partial x_{n+1}}{\partial x_j}\cdot \frac{\dif x_j}{\dif x_1}(z)| \\
    &\leq \sum\limits_{j < n+1}|\frac{\partial x_{n+1}}{\partial x_j}\cdot \frac{\dif x_j}{\dif x_1}(y) - \frac{\partial x_{n+1}}{\partial x_j}\cdot \frac{\dif x_j}{\dif x_1}(z)|\\
    &\leq \sum\limits_{j<n+1}(|\frac{\partial x_{n+1}}{\partial x_j}(y) \frac{\dif x_j}{\dif x_1}(y) - \frac{\partial x_{n+1}}{\partial x_j}(y) \frac{\dif x_j}{\dif x_1}(z)|\\ &\ \ \ \ \ \ \ \ \ \ \ \ \ \ \ \ +|\frac{\partial x_{n+1}}{\partial x_j}(y) \frac{\dif x_j}{\dif x_1}(z) - \frac{\partial x_{n+1}}{\partial x_j}(z) \frac{\dif x_j}{\dif x_1}(z)|)\\
    &\leq \sum\limits_{j<n+1} (K \mathcal O(((j-1)J)^{j-1}S)|y-z| \\ &\ \ \ \ \ \ \ \ \ \ \ \ \ \ \ \ + S |y-z|\mathcal O(((j-1)J)^{j-1}) )\\
    &\leq n (K\mathcal O(((n-1)J)^{n-1}S) + S\mathcal O(((n-1)J)^{n-1}) )|y-z|\\
    &=\mathcal O((nJ)^{n}S)|y-z|
\end{align*}
By mathematical induction, we can get the conclusion
\[|\dfrac{\dif x_i}{\dif x_1}(y) - \dfrac{\dif x_i}{\dif x_1}(z)|\leq \mathcal O(((i-1)J)^{i-1}S)|y-z|\]
\end{proof}
\begin{lemma}\label{calM}
    When assumption 1-4 hold, we have
    \[|\dfrac{\dif f_1}{\dif x_1}(y) - \dfrac{\dif f_1}{\dif x_1}(z)|\leq \mathcal O(n^n J^{n-1}(SN + M))|y-z| = \mathcal O(n^nJ^{n+1})|y-z|\]
    for any $y$ and $z$.
\end{lemma}
\begin{proof}\label{a3}
    We can use the triangle inequality to unravel the difference:
    \begin{align*}
        |\dfrac{\dif f_1}{\dif x_1}(y) - \dfrac{\dif f_1}{\dif x_1}(z)| &= |\sum\limits_{j=1}^n\frac{\partial f_1}{\partial x_j}\cdot \frac{\dif x_j}{\dif x_1}(y) - \sum\limits_{j=1}^n\frac{\partial f_1}{\partial x_j}\cdot \frac{\dif x_j}{\dif x_1}(z)|\\
        &\leq \sum\limits_{j=1}^n |\frac{\partial f_1}{\partial x_j}\cdot \frac{\dif x_j}{\dif x_1}(y) -\frac{\partial f_1}{\partial x_j}\cdot \frac{\dif x_j}{\dif x_1}(z)|\\
        &\leq \sum\limits_{j=1}^n (|\frac{\partial f_1}{\partial x_j}(y) \frac{\dif x_j}{\dif x_1}(y) - \frac{\partial f_1}{\partial x_j}(y) \frac{\dif x_j}{\dif x_1}(z)| \\ &\ \ \ \ \ \ \ \ \ \ \ \ \ \ \ \ + |\frac{\partial f_1}{\partial x_j}(y) \frac{\dif x_j}{\dif x_1}(z) - \frac{\partial f_1}{\partial x_j}(z) \frac{\dif x_j}{\dif x_1}(z)|)\\
        &\leq \sum\limits_{j=1}^n (N\mathcal O(((j-1)J)^{j-1}S)|y-z| + M|y-z|  \mathcal O(((j-1)J)^{j-1}))\\
        &\leq n(N\mathcal O(((n-1)J)^{n-1}S) + M  \mathcal O(((n-1)J)^{n-1})))|y-z|\\
        &\leq \mathcal O(n^n J^{n-1}(SN + M))|y-z|\\
        &=\mathcal O(n^nJ^{n+1})|y-z|
     \end{align*}
\end{proof}
According to the three lemmas, we can prove the general conclusion given above:
\begin{theorem}
    When the functions $f_i\ (\forall i)$ in $n$-level optimization and $x_j\ (\forall j >1)$ derived from them can satisfy the assumptions 1-4, we will have
    \[\mathbb E (\dfrac{\dif f_1}{\dif x_1})^2\leq \dfrac{F_0}{N}(\beta - \frac{C\beta^2n^nJ^{n+1}}{2})^{-1}\]
\end{theorem}
\begin{proof}\label{a4}
    For simplicity, we use $x^i$ to represent $x^i(x_1^i)=(x_1^i, x_2^i, \cdots, x_n^i)$ in the $i$-th round, where $x_2^i,\cdots,x_n^i$ are all functions of $x_1^i$, just like what is said in the system of equations \ref{equations}. From lemma \ref{calM}, we know that there is a constant C such that
    \begin{equation*}
        |\dfrac{\dif^2f_1}{\dif x_1^2}| \leq Cn^nJ^{n+1}
    \end{equation*}
    We can also compute the difference of $f_1$ between two rounds:
    \begin{align*}
        &f_1(x^{i+1}) - f_1(x^i)\\
        =&\frac{\dif f_1}{\dif x_1}(x^i)(x_1^{i+1} - x_1^i) + \frac{1}{2}\dfrac{\dif^2f_1}{\dif x_1^2}(\xi^i)(x_1^{i+1} - x_1^i)^2\\
        \leq&-\beta(\frac{\dif f_1}{\dif x_1}(x^i))^2 + \frac{1}{2}\beta^2 (\frac{\dif f_1}{\dif x_1}(x^i))^2 Cn^nJ^{n+1}
    \end{align*}
    Sum the both sides and when $\beta < \frac{2}{Cn^nJ^{n+1}}$, we can get
    \begin{align*}
     &\sum_{i=0}^{N-1}(\dfrac{\dif f_1}{\dif x_1}(x^i))^2\\
        \leq& (\beta - \frac{C\beta^2n^nJ^{n+1}}{2})^{-1}(f_1(x^0) - f_1(x^N))\\
        \leq& (\beta - \frac{C\beta^2n^nJ^{n+1}}{2})^{-1}f_1(x^0)\\
        =&F_0(\beta - \frac{C\beta^2n^nJ^{n+1}}{2})^{-1}
    \end{align*}
    By Central Limit Theorem, we have
    \[\mathbb E (\dfrac{\dif f_1}{\dif x_1})^2\sim\frac{1}{N}\sum_{i=0}^{N-1}(\dfrac{\dif f_1}{\dif x_1}(x^i))^2\leq \dfrac{F_0}{N}(\beta - \frac{C\beta^2n^nJ^{n+1}}{2})^{-1}\]
\end{proof}
From the above proof, note that:
\begin{enumerate}
    \item 
    These four assumptions about the boundedness of derivatives can also be replaced by the boundedness of the lipschitz constants for $f_1$ and $x_j\ (\forall j > 1)$ and their corresponding derivatives.
    \item 
    In fact, $x_j\ (\forall j > 1)$ depends on $f_i\ (\forall i > 1)$. Due to the structure of $n$-level optimization, the various derivatives of $x_j$ can be bounded by the lipschitz constants of various derivatives of $f_i$, which may lead to more optimal bounds. In fact, the expectation of $(\dfrac{\dif f_1}{\dif x_1})^2$ can be bounded by the lipschitz constants of the each order derivative of $f_i\ (\forall i)$.
    \item 
    This bound is intuitive. Note that $\beta < \frac{2}{Cn^nJ^{n+1}}$, so we can choose $\beta = \frac{1}{Cn^nJ^{n+1}}$, it means 
    \[\mathbb E (\dfrac{\dif f_1}{\dif x_1})^2\leq \dfrac{2Cn^nJ^{n+1}F_0}{N}\]
    As the number of layers $n$ increases, more rounds $N$ is needed to update $x_1$ to achieve the convergence.
    \item 
    This theorem explains that in general, when the domain of $\mathbf x_1$ is not a compact convex set, the algorithm we use still converges, although it may not necessarily converge to the optimal value.
\end{enumerate}


\end{document}